\documentclass{article}

\usepackage{microtype}
\usepackage{graphicx}
\usepackage{subfigure}
\usepackage{booktabs} 

\usepackage{hyperref}

\usepackage[accepted]{icml2019}

\usepackage{amsfonts}
\usepackage{amsmath}
\usepackage{amssymb}
\usepackage{amsthm}
\usepackage[algo2e,noend]{algorithm2e}

\mathchardef\mhyphen="2D 
\DeclareMathOperator*{\Val}{Val}
\newcommand{\Tau}{\mathcal{T}}
\newcommand{\Taub}{\bs{\Tau}}
\newcommand{\bs}{\boldsymbol}
\newcommand\independent{\protect\mathpalette{\protect\independenT}{\perp}}
\def\independenT#1#2{\mathrel{\rlap{$#1#2$}\mkern2mu{#1#2}}}
\newcommand{\ent}{\mathcal{H}}
\newcommand{\rep}{\mathcal{R}}

\newcommand{\lb}{\mathcal{L}}
\newcommand{\taub}{\bs{\tau}}
\newtheorem{theorem}{Theorem}
\newcommand\numberthis{\addtocounter{equation}{1}\tag{\theequation}}

\newcommand\blfootnote[1]{%
  \begingroup
  \renewcommand\thefootnote{}\footnote{#1}%
  \addtocounter{footnote}{-1}%
  \endgroup
}

\icmltitlerunning{Maximum Entropy-Regularized Multi-Goal Reinforcement Learning}

\begin{document}

\twocolumn[
\icmltitle{Maximum Entropy-Regularized Multi-Goal Reinforcement Learning}




\begin{icmlauthorlist}
\icmlauthor{Rui Zhao}{LMU,Siemens}
\icmlauthor{Xudong Sun}{LMU}
\icmlauthor{Volker Tresp}{LMU,Siemens}
\end{icmlauthorlist}

\icmlaffiliation{LMU}{Faculty of Mathematics, Informatics and Statistics, Ludwig Maximilian University of Munich, Munich, Bavaria, Germany}
\icmlaffiliation{Siemens}{Siemens AG, Munich, Bavaria, Germany}

\icmlcorrespondingauthor{Rui Zhao}{zhaorui.in.germany@gmail.com} 

\icmlkeywords{Deep RL, Maximum Entropy, Goal-Conditioned Policy, Hindsight Experience}

\vskip 0.3in
]



\printAffiliationsAndNotice{}  

\begin{abstract}

In Multi-Goal Reinforcement Learning, an agent learns to achieve multiple goals with a goal-conditioned policy. During learning, the agent first collects the trajectories into a replay buffer, and later these trajectories are selected randomly for replay. However, the achieved goals in the replay buffer are often biased towards the behavior policies. From a Bayesian perspective, when there is no prior knowledge about the target goal distribution, the agent should learn uniformly from diverse achieved goals. Therefore, we first propose a novel multi-goal RL objective based on weighted entropy. This objective encourages the agent to maximize the expected return, as well as to achieve more diverse goals.  Secondly, we developed a maximum entropy-based prioritization framework to optimize the proposed objective. For evaluation of this framework, we combine it with Deep Deterministic Policy Gradient, both with or without Hindsight Experience Replay. On a set of multi-goal robotic tasks of OpenAI Gym, we compare our method with other baselines and show promising improvements in both performance and sample-efficiency. 

\end{abstract}

\section{Introduction}

Reinforcement Learning (RL) \cite{sutton1998reinforcement} combined with Deep Learning (DL) \cite{goodfellow2016deep,zhao2017two} has led to great successes in various tasks, such as playing video games \cite{mnih2015human}, challenging the World Go Champion \cite{silver2016mastering}, conducting goal-oriented dialogues \cite{bordes2016learning, zhao2018improving, zhao2018learning, zhao2018efficient}, and learning autonomously to accomplish different  robotic tasks \cite{ng2006autonomous, peters2008reinforcement, levine2016end, chebotar2017path, andrychowicz2017hindsight}.

One of the biggest challenges in RL is to make the agent learn efficiently in applications with sparse rewards.
To tackle this challenge, \citet{lillicrap2015continuous} developed the Deep Deterministic Policy Gradient (DDPG), which enables the agent to learn continuous control, such as manipulation and locomotion. \citet{schaul2015universal} proposed Universal Value Function Approximators (UVFAs), which generalize not just over states, but also over goals, and extend value functions to multiple goals. Furthermore, to make the agent learn faster in sparse reward settings, \citet{andrychowicz2017hindsight} introduced Hindsight Experience Replay (HER), which encourages the agent to learn from the goal-states it has achieved. The combined use of DDPG and HER allows the agent to learn to accomplish more complex robot manipulation tasks. However, there is still a huge gap between the learning efficiency of humans and RL agents.  
In most cases, an RL agent needs millions of samples before it is able to solve the tasks, while humans only need a few samples \cite{mnih2015human}. \blfootnote{This paper is based on our 2018 NeurIPS Deep RL workshop paper \cite{zhao2019curiosity}.}

\begin{figure*}
	\centering
	\includegraphics[width=5.5 in]{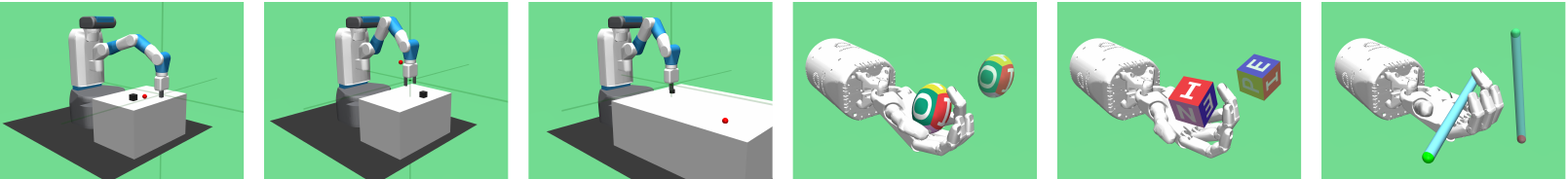}
	\caption{Robot arm Fetch and Shadow Dexterous hand environment: 
	\texttt{FetchPush}, \texttt{FetchPickAndPlace}, \texttt{FetchSlide},
	\texttt{HandManipulateEgg}, \texttt{HandManipulateBlock}, and \texttt{HandManipulatePen}.}
	\label{fig:fetchhand6env}
\end{figure*}
In previous works, the concept of maximum entropy has been used to encourage exploration during training \cite{williams1991function,mnih2015human,wu2016training}.
Recently,  \citet{haarnoja2017reinforcement} introduced Soft-Q Learning, which learns a deep energy-based policy by evaluating the maximum entropy of actions for each state. Soft-Q Learning encourages the agent to learn all the policies that lead to the optimum \cite{levine2018reinforcement}.  Furthermore, Soft Actor-Critic \cite{haarnoja2018soft} demonstrated a better performance while showing compositional ability and robustness of the maximum entropy policy in locomotion \cite{haarnoja2018latent} and robot manipulation tasks \cite{haarnoja2018composable}. The agent aims to maximize the expected reward while also maximizing the entropy to succeed at the task while acting as randomly as possible. Based on maximum entropy policies, \citet{eysenbach2018diversity} showed that the agent is able to develop diverse skills solely by maximizing an information theoretic objective without any reward function. For multi-goal and multi-task learning \cite{caruana1997multitask}, the diversity of training sets helps the agent transfer skills to unseen goals and tasks \cite{pan2010survey}.  The variability of training samples mitigates overfitting and helps the model to better generalize \cite{goodfellow2016deep}.  In our approach, we combine maximum entropy with multi-goal RL to help the agent to achieve unseen goals by learning uniformly from diverse achieved goals during training.  

We observe that during experience replay the uniformly sampled trajectories are biased towards the behavior policies, with respect to the achieved goal-states. 
Consider training a robot arm to reach a certain point in a space. At the beginning, the agent samples trajectories using a random policy. The sampled trajectories are centered around the initial position of the robot arm. Therefore, the distribution of achieved goals, i.e., positions of the robot arm, is similar to a Gaussian distribution around the initial position, which is non-uniform. Sampling from such a distribution is biased towards the current policies. From a Bayesian point of view \cite{murphy2012machine}, the agent should learn uniformly from these achieved goals, when there is no prior knowledge of the target goal distribution. 

To correct this bias, we propose a new objective which combines maximum entropy and the multi-goal RL objective. This new objective uses entropy as a regularizer to encourage the agent to traverse diverse goal-states. Furthermore, we derive a safe lower bound for optimization. To optimize this surrogate objective, we implement maximum entropy-based prioritization as a simple yet effective solution.


\section{Preliminary}
\label{sec:background}  
\subsection{Settings}
\label{sec:environments}
\textbf{Environments:} 
We consider multi-goal reinforcement learning tasks, like the robotic simulation scenarios provided by OpenAI Gym \cite{plappert2018multi}, where six challenging tasks are used for evaluation, including push, slide, pick \& place with the robot arm, as well as hand manipulation of the block, egg, and pen, as shown in Figure~\ref{fig:fetchhand6env}.
Accordingly, we define the following terminologies for this specific kind of multi-goal scenarios. 

\textbf{Goals:} 
The goals $g$ are the desired positions and the orientations of the object. 
Specifically, we use $g^e$, with $e$ standing for environment, to denote the real goal which serves as the input from the environment, in order to distinguish it from the achieved goal used in Hindsight settings \cite{andrychowicz2017hindsight}.  
Note that in this paper we consider the case where the goals can be represented by states, which leads us to the concept of achieved goal-state $g^s$, with details explained below. 

\textbf{States, Goal-States and Achieved Goals:} 
The state $s$ consists of two sub-vectors, the achieved goal-state $s^g$, which represents the position and orientation of the object being manipulated, and the context state $s^c$, i.e.\ $s = (s^g \| s^c)$, where $\|$ denotes concatenation.

In our case, we define $g^s=s^g$ to represent an achieved goal that has the same dimension as the real goal $g^e$ from the environment. 
The context state $s^c$ contains the rest information about the state, including the linear and angular velocities of all robot joints and of the object. The real goals $g^e$ can be substituted by the achieved goals $g^s$ to facilitate learning.
This goal relabeling technique was proposed by \citet{andrychowicz2017hindsight} as Hindsight Experience Replay.

\textbf{Achieved Goal Trajectory:} 
A trajectory consisting solely of goal-states is represented as $\bs{\tau}^g$.
We use $\taub^g$ to denote all the achieved goals in the trajectory $\bs{\tau}$, i.e., $\taub^g = (g^{s}_0, ... , g^{s}_T)$.

\textbf{Rewards:} We consider sparse rewards $r$. There is a tolerated range between the desired goal-states and the achieved goal-states.
If the object is not in the tolerated range of the real goal, the agent receives a reward signal -$1$ for each transition; otherwise, the agent receives a reward signal $0$.

\textbf{Goal-Conditioned Policy:} 
In multi-goal settings, the agent receives the environmental goal $g^e$ and the state input $s = (s^g \| s^c)$. 
We want to train a goal-conditioned policy to effectively generalize its behavior to different environmental goals $g^e$.

\subsection{Reinforcement Learning}
\label{sec:rl}

We consider an agent interacting with an environment. We assume the environment is fully observable, including a set of state $\mathcal{S}$, a set of action $\mathcal{A}$, a distribution of initial states $p(s_0)$, transition probabilities $p(s_{t+1} \mid s_t, a_t)$, a reward function $r$: $\mathcal{S} \times \mathcal{A} \rightarrow \mathbb{R}$, and a discount factor $\gamma \in [0,1]$. 

\textbf{Deep Deterministic Policy Gradient:}
For continuous control tasks, the Deep Deterministic Policy Gradient (DDPG) shows promising performance, which is essentially an off-policy actor-critic method \cite{lillicrap2015continuous}. 

\textbf{Universal Value Function Approximators:}
For multi-goal continuous control tasks, DDPG can be extended by Universal Value Function Approximators (UVFA) \cite{schaul2015universal}. UVFA essentially generalizes the Q-function to multiple goal-states, where the Q-value depends not only on the state-action pairs, but also on the goals.

\textbf{Hindsight Experience Replay:}
For robotic tasks, if the goal is challenging and the reward is sparse, the agent could perform badly for a long time before learning anything. 
Hindsight Experience Replay (HER) encourages the agent to learn from whatever goal-states it has achieved. 
\citet{andrychowicz2017hindsight} show that HER makes training possible in challenging robotic tasks via goal relabeling, i.e., randomly substituting real goals with achieved goals.

\subsection{Weighted Entropy} 
\citet{guiacsu1971weighted} proposed weighted entropy, which is an extension of Shannon entropy. The definition of weighted entropy is given as
\begin{equation}
\ent^w_p = -\sum_{k=1}^K w_kp_k \log p_k,
\end{equation}
where $w_k$ is the weight of the elementary event and $p_k$ is the probability of the elementary event.
\section{Method}
\label{sec:method}

In this section, we formally describe our method, including the mathematical derivation of the Maximum Entropy-Regularized Multi-Goal RL objective and the Maximum Entropy-based Prioritization framework. 

\subsection{Multi-Goal RL}
\label{sec:sub:multi-goal-rl}
In this paper, we consider multi-goal RL as goal-conditioned policy learning \cite{schaul2015universal,andrychowicz2017hindsight,rauber2017hindsight,plappert2018multi}.
We denote random variables by upper case letters and the values of random variables by corresponding lower case letters. For example, let $\Val(X)$ denote the set of valid values to a random variable $X$, and let $p(x)$ denote the probability function of random variable $X$.

Consider that an agent receives a goal $g^e \in \Val(G^e)$ at the beginning of the episode. The agent interacts with the environment for $T$ timesteps. 
At each timestep $t$, the agent observes a state $s_t \in \Val(S_t)$ and performs an action $a_t \in \Val(A_t)$.
The agent also receives a reward conditioned on the input goal $r(s_t, g^e) \in \mathbb{R}$.

We use $\bs{\tau} = s_1, a_1, s_2, a_2, \ldots, s_{T-1}, a_{T-1}, s_{T}$ to denote a trajectory, where $\taub \in 
\Val(\Taub)$. 
We assume that the probability $p(\bs{\tau} \mid g^e, \bs{\theta})$ of trajectory $\bs{\tau}$, given goal $g^e$ and a policy parameterized by $\bs{\theta} \in \Val(\bs{\Theta})$, is given as
\begin{equation*}
 p(\bs{\tau} \mid g^e, \bs{\theta}) = p(s_1) \prod_{t=1}^{T-1} p(a_t \mid s_t, g^e, \bs{\theta}) p(s_{t+1} \mid s_t, a_t).
\label{eq:trajectory}
\end{equation*}

The transition probability $p(s_{t+1} \mid s_t, a_t)$ states that the probability of a state transition given an action is independent of the goal, and we denote it with $S_{t+1} \independent G^e \mid S_{t}, A_{t}$. For every $\bs{\tau}, g^e,$ and $\bs{\theta}$, we also assume that $p(\bs{\tau} \mid g^e, \bs{\theta})$ is non-zero. The expected return of a policy parameterized by $\bs{\theta}$ is given as
\begin{equation}
\begin{split}
\eta(\bs{\theta}) &=\mathbb{E}\left[ \sum_{t = 1}^T r(S_t, G^e) \mid \bs{\theta} \right]  \\
							&=  \sum_{g^e} p(g^e) \sum_{\bs{\tau}} p(\bs{\tau} \mid g^e, \bs{\theta}) \sum_{t = 1}^T r(s_t, g^e).
\end{split}
\label{eq:objective}
\end{equation}
Off-policy RL methods use experience replay \cite{lin1992self,mnih2015human} to leverage bias over variance and potentially improve sample-efficiency.
In the off-policy case, the objective, Equation~(\ref{eq:objective}),  is given as
\begin{equation}
\begin{split}
\eta^{\rep}(\bs{\theta}) 	
							&=  \sum_{\bs{\tau},\ g^e} p_{\rep}(\bs{\tau},g^e \mid \bs{\theta}) \sum_{t = 1}^T r(s_t, g^e),
\end{split}
\label{eq:objective-offpolicy}
\end{equation}
where $\rep$ denotes the replay buffer. 
Normally, the trajectories $\bs{\tau}$ are randomly sampled from the buffer.
However, we observe that the trajectories in the replay buffer are often imbalanced with respect to the achieved goals $\taub^g$.
Thus, we propose Maximum Entropy-Regularized Multi-Goal RL to improve performance.

\subsection{Maximum Entropy-Regularized Multi-Goal RL}
\label{sec:sub:maxent-multi-goal-rl}
In multi-goal RL, we want to encourage the agent to traverse diverse goal-state trajectories, and at the same time, maximize the expected return. This is like maximizing the empowerment \cite{mohamed2015variational} of an agent attempting to achieve multiple goals.
We propose the reward-weighted entropy objective for multi-goal RL, which is given as
\begin{equation}\begin{split}
\eta^{\ent}(\bs{\theta}) &=\ent^w_p(\Taub^g) \\ &= \mathbb{E}_p\left[\log\frac{1}{p(\taub^g)} \sum_{t = 1}^T r(S_t, G^e) \mid
\bs{\theta}\right] 
\end{split}
\label{eq:entropy-objective-traj}.
\end{equation} 
For simplicity, we use $p(\taub^g)$ to represent $\sum_{g^e}p_{\rep}(\bs{\tau}^g, g^e \mid \bs{\theta})$, which is the occurrence probability of the goal-state trajectory $\taub^g$. The expectation is calculated based on $p(\taub^g)$ as well, so the proposed objective is the weighted entropy \cite{guiacsu1971weighted, Kelbert2017} of $\taub^g$, which we denote as $\ent^w_p(\Taub^g)$, where the weight $w$ is the accumulated reward $\sum_{t = 1}^T r(s_t, g^e)$ in our case. 

The objective function, Equation~(\ref{eq:entropy-objective-traj}), has two interpretations. 
The first interpretation is to maximize the weighted expected return, where the rare trajectories have larger weights. Note that when all trajectories occur uniformly, this weighting mechanism has no effect. The second interpretation is to maximize a reward-weighted entropy, where the more rewarded trajectories have higher weights. This objective encourages the agent to learn how to achieve diverse goal-states, as well as to maximize the expected return.

In Equation~(\ref{eq:entropy-objective-traj}), the weight, $\log \left(1/p(\taub^g)\right)$, is unbounded, which makes the training of the universal function approximator unstable. Therefore, we propose a safe surrogate objective, $\eta^{\lb}$, which is essentially a lower bound of the original objective.

\subsection{Surrogate Objective}
\label{sec:sub:surrogate}

To construct the safe surrogate objective, we sample the trajectories from the replay buffer with a proposal distribution, $q(\taub^g) = \frac{1}{Z}p(\taub^g)\left(1-p(\taub^g)\right)$. $p(\taub^g)$ represents the distribution of the goal trajectories in the replay buffer. The surrogate objective is given in Theorem~\ref{th:lower-bound}, which is proved to be a lower bound of the original objective, Equation~(\ref{eq:entropy-objective-traj}).

\begin{theorem}
	\label{th:lower-bound}
	The surrogate $\eta^{\lb}(\bs{\theta})$ is a lower bound of the objective function $\eta^{\ent}(\bs{\theta})$, i.e.,  $\eta^{\lb}(\bs{\theta}) < \eta^{\ent}(\bs{\theta})$, where

	\begin{equation}
	\begin{split}
	\eta^{\ent}(\bs{\theta})   &=\ent^w_p(\Taub^g) \\
											&= \mathbb{E}_p\left[\log\frac{1}{p(\taub^g)} \sum_{t = 1}^T r(S_t, G^e) \mid \bs{\theta}\right] 
	\end{split}
	\end{equation} 
	
	\begin{equation}
	\eta^{\lb}(\bs{\theta}) = Z \cdot \mathbb{E}_q \left[\sum_{t = 1}^T r(S_t, G^e) \mid \bs{\theta} \right]
	\label{eq:th1-lb}
	\end{equation}
	
	\begin{equation}
	q(\taub^g) = \frac{1}{Z}p(\taub^g)\left(1-p(\taub^g)\right)
	\end{equation}
	
	$Z$ is the normalization factor for $q(\taub^g)$.
	$\ent^w_p(\Taub^g)$ is the weighted entropy \cite{guiacsu1971weighted, Kelbert2017}, where the weight is the accumulated reward $\sum_{t = 1}^T r(S_t, G^e)$, in our case. 
	
\end{theorem}
\begin{proof}
	See Appendix.
\end{proof}

\subsection{Prioritized Sampling}
\label{sec:sub:prio}
To optimize the surrogate objective, Equation~(\ref{eq:th1-lb}), we cast the optimization process into a prioritized sampling framework.  At each iteration, we first construct the proposal distribution $q(\taub^g)$, which has a higher entropy than $p(\taub^g)$. This ensures that the agent learns from a more diverse goal-state distribution. In Theorem \ref{th:higher-entropy}, we prove that the entropy with respect to $q(\taub^g)$ is higher than the entropy with respect to $p(\taub^g)$.

\begin{theorem}
\label{th:higher-entropy}
Let the probability density function of goals in the replay buffer be 
\begin{equation}
p(\taub^g),\text{where}\ p(\taub^g_i) \in (0,1)\ \text{and}\ \sum_{i=1}^{N} p(\taub^g_i)=1.
\end{equation}
Let the proposal probability density function be defined as 
\begin{equation}
q(\taub^g_i) = \frac{1}{Z} p(\taub^g_i) \left( 1 - p(\taub^g_i) \right ) ,\ \text{where}\ \sum_{i=1}^N q(\taub^g_i)=1.
\end{equation}
Then, the proposal goal distribution has an equal or higher entropy 
\begin{equation}
\ent_q({\Taub^g}) - \ent_p (\Taub^g) \geq 0.
\end{equation}
\end{theorem}
\begin{proof}
See Appendix.
\end{proof}

\subsection{Estimation of Distribution}
\label{sec:sub:prob}
To optimize the surrogate objective with prioritized sampling, we need to know the probability distribution of a goal-state trajectory $p(\taub^g)$.
We use a Latent Variable Model (LVM) \cite{murphy2012machine} to model the underlying distribution of $p(\taub^g)$, since LVM is suitable for modeling complex distributions. 

\begin{figure*}
    \centering
    \begin{minipage}{0.45\linewidth}
        \includegraphics[width=\linewidth]{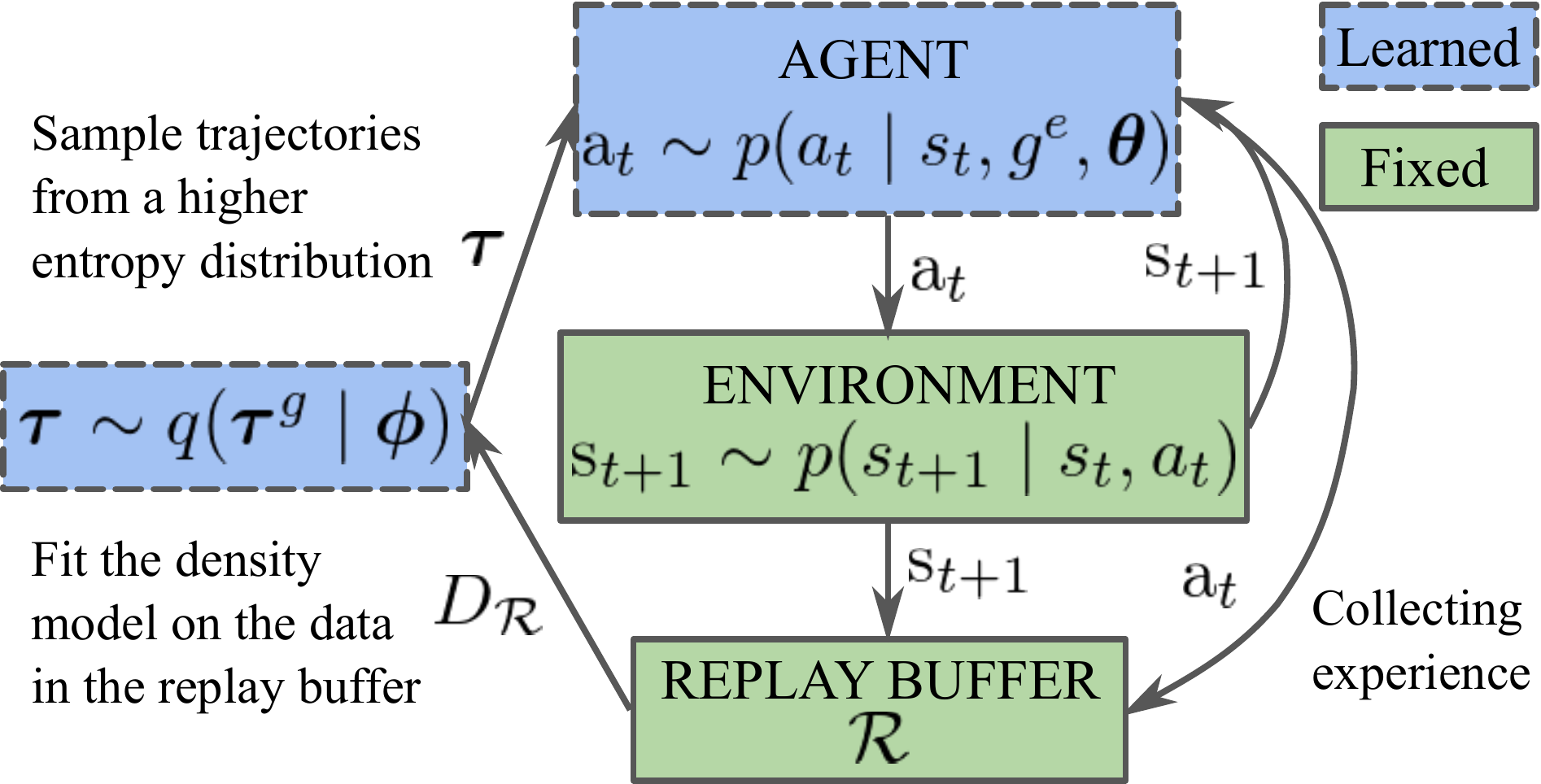}
    \end{minipage} \hfill
    \begin{minipage}{0.54\linewidth}
    
    \begin{algorithm}[H]
    \small
    \DontPrintSemicolon
    \SetAlgoLined
    \While{not converged}{
        Sample goal $g^e \sim p(g^e)$ and initial state $s_0 \sim p(s_0)$\\
        \For{$steps\_per\_epoch$}{
    		\For{$steps\_per\_episode$}{
            		Sample action $a_t \sim p(a_t \mid s_t, g^e, \bs{\theta})$ from behavior policy.\\
            		Step environment: $s_{t+1} \sim p(s_{t+1} \mid s_t, a_t)$.\\
            		Update replay buffer $\rep$.\\
            		Construct prioritized sampling distribution:\\ $q(\taub^g) \propto (1- p(\taub^g \mid \bs{\phi})) p(\taub^g)$ with higher $\ent_q (\Taub^g)$.\\
				Sample trajectories $\bs{\tau} \sim q(\taub^g \mid \bs{\phi})$ \\
            		Update policy ($\bs{\theta}$) to max.\ $\mathbb{E}_{q}\left[  r(S, G) \right]$ via DDPG, HER.
            	}
            	Update density model ($\bs{\phi}$).
        }
    } 
    \caption{Maximum Entropy-based Prioritization (MEP)}\label{algo:complete}
    \end{algorithm}
    \end{minipage}
    \caption{\textbf{MEP Algorithm}:
    We update the density model to construct a higher entropy distribution of achieved goals and update the agent with the more diversified training distribution. \label{fig:mep_model}}
\end{figure*}

Specifically, we use $p(\taub^g \mid z_{k})$ to denote the latent-variable-conditioned goal-state trajectory distribution, which we assume to be Gaussians. 
$z_{k}$ is the $k$-th latent variable, where $k \in \{ 1, ..., K\}$ and $K$ is the number of the latent variables. The resulting model is a Mixture of Gaussians(MoG), mathematically,
\begin{equation}
p(\taub^g \mid \bs{\phi})= \frac{1}{Z} \sum_{i=k}^K c_k \mathcal{N}(\taub^g | \boldsymbol{\mu}_k, \boldsymbol{\Sigma}_k),
\label{eq:lower-bound}
\end{equation}
where each Gaussian, $\mathcal{N}(\taub^g | \boldsymbol{\mu}_k, \boldsymbol{\Sigma}_k)$, has its own mean $\boldsymbol{\mu}_k$ and covariance $\boldsymbol{\Sigma}_k$, $c_k$ represents the mixing coefficients, and $Z$ is the partition function. The model parameter $\bs{\phi}$ includes all mean $\boldsymbol{\mu}_i$, covariance $\boldsymbol{\Sigma}_i$, and mixing coefficients $c_k$.

In prioritized sampling, we use the complementary predictive density of a goal-state trajectory $\taub^g$ as the priority, which is given as 
\begin{equation}
\bar{p}(\taub^g \mid \bs{\phi}) \propto 1- p(\taub^g \mid \bs{\phi}) .
\label{eq:density_bar}
\end{equation}
The complementary density describes the likelihood that a goal-state trajectory $\taub^g$ occurs in the replay buffer. 
A high complementary density corresponds to a rare occurrence of the goal trajectory.
We want to over-sample these rare goal-state trajectories during replay to increase the entropy of the training distribution.
Therefore, we use the complementary density to construct the proposal distribution as a joint distribution
\begin{equation}
\begin{split}
q(\taub^g) &\propto \bar{p}(\taub^g \mid \bs{\phi}) p(\taub^g) \\ 
				  &\propto (1- p(\taub^g \mid \bs{\phi})) p(\taub^g) \\
				  &\approx p(\taub^g) - p(\taub^g)^2.
\label{eq:contruct_proposal}
\end{split}
\end{equation}

\subsection{Maximum Entropy-Based Prioritization}
\label{sec:sub:maxent-pri}
With prioritized sampling, the agent learns to maximize the return of a more diverse goal distribution.
When the agent replays the samples, it first ranks all the trajectories with respect to their proposal distribution $q(\taub^g)$, and then uses the ranking number directly as the probability for sampling. 
This means that rare goals have high ranking numbers and, equivalently, have higher priorities to be replayed. 
Here, we use the ranking instead of the density.
The reason is that the rank-based variant is more robust since it is neither affected by outliers nor by density magnitudes. 
Furthermore, its heavy-tail property also guarantees that samples will be diverse \cite{schaul2015prioritized}. 
Mathematically, the probability of a trajectory to be replayed after the prioritization is:
\begin{equation}
q(\taub^g_{i}) = \frac{\mathrm{rank}(q(\taub^g_{i}))}{\sum_{n=1}^N \mathrm{rank}(q(\taub^g_{n}))},
\label{eq:rank}
\end{equation}
where $N$ is the total number of trajectories in the replay buffer and $\mathrm{rank}(\cdot)$ is the ranking function.

\begin{figure*}
	\centering
	\includegraphics[width=5.5 in]{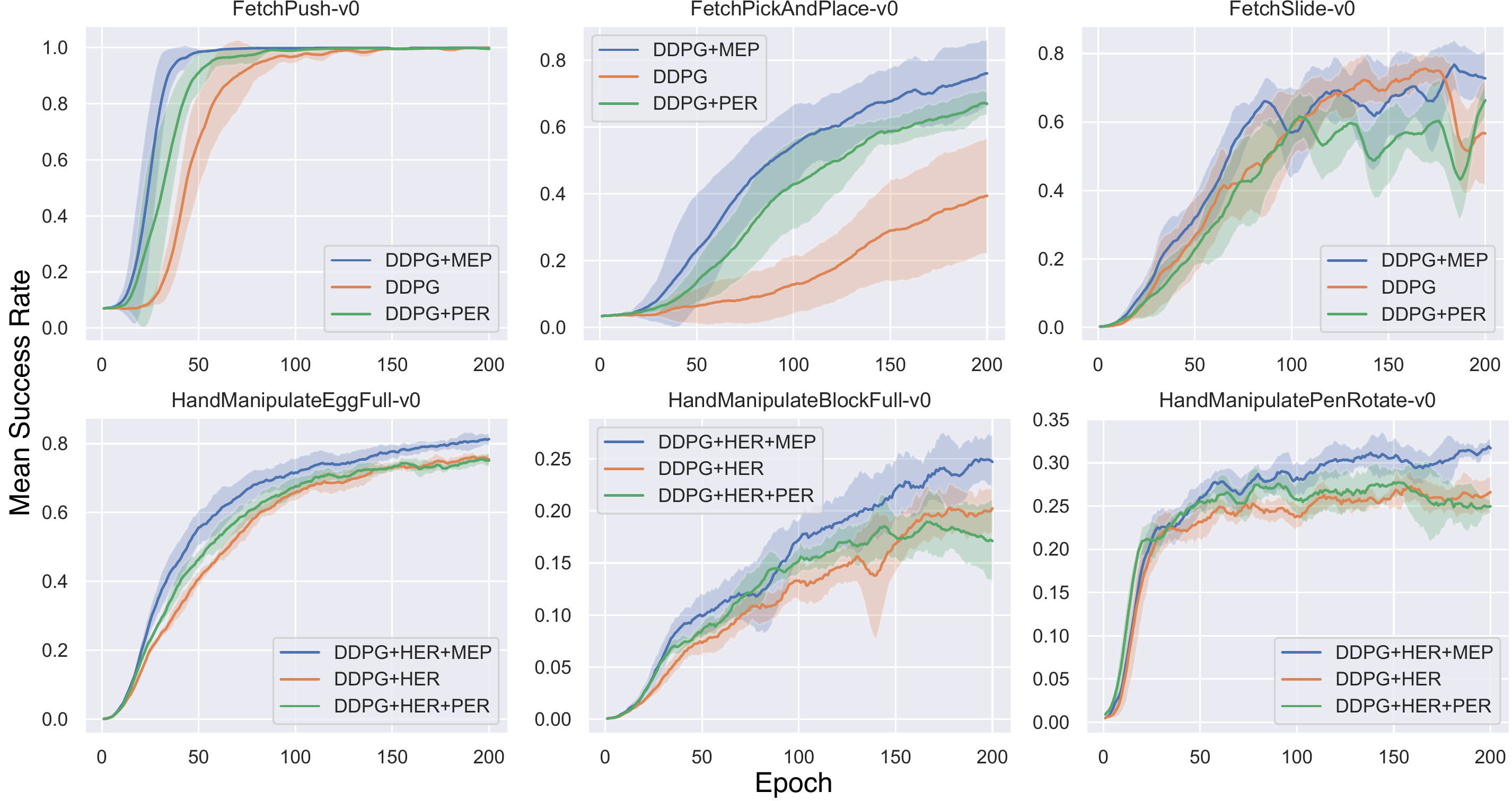}
	\caption{Mean success rate with standard deviation in all six robot environments}
	\label{fig:fig_accuracy}
\end{figure*}

We summarize the complete training algorithm in Algorithm~\ref{algo:complete} and in Figure~\ref{fig:mep_model}.
In short, we propose Maximum Entropy-Regularized Multi-Goal RL (Section~\ref{sec:sub:maxent-multi-goal-rl}) to enable RL agents to learn more efficiently in multi-goal tasks (Section~\ref{sec:sub:multi-goal-rl}).
We integrate a goal entropy term into the normal expected return objective.
To maximize the objective, Equation~(\ref{eq:entropy-objective-traj}), we derive a surrogate objective in Theorem~\ref{th:lower-bound}, i.e., a lower bound of the original objective. 
We use prioritized sampling based on a higher entropy proposal distribution at each iteration and utilize off-policy RL methods to maximize the expected return.
This framework is implemented as Maximum Entropy-based Prioritization (MEP). 

\begin{table*}[!ht]
\centering
\caption{Mean success rate (\%) and training time (hour) for all six environments}
\begin{tabular}{p{2.9cm} p{1.6cm}rp{1.6cm}r  p{1.6cm}rp{1.6cm}r p{1.6cm}rp{1.6cm}r}
\toprule
 							& \multicolumn{2}{c}{Push} 			& \multicolumn{2}{c}{Pick \& Place} 		& \multicolumn{2}{c}{Slide} 		\\
 							\cmidrule(lr){2-3} 			   	  			\cmidrule(lr){4-5} 								 	\cmidrule(lr){6-7}
Method     				& success   				& time  			& success   				& time 					& success   				& time 		\\
\midrule
DDPG       				& 99.90\% 				& 5.52h 		& 39.34\%  				& 5.61h  				& 75.67\% 				& 5.47h		\\
DDPG+PER 			& 99.94\% 				& 30.66h 		& 67.19\%  				& 25.73h  				& 66.33\% 				& 25.85h	\\
DDPG+MEP   			& \textbf{99.96\%}	& 6.76h 		& \textbf{76.02\%}  & 6.92h  				& \textbf{76.77\%} 	& 6.66h		\\
\midrule
\midrule
 							& \multicolumn{2}{c}{Egg} 			& \multicolumn{2}{c}{Block} 					& \multicolumn{2}{c}{Pen} 		\\
 							\cmidrule(lr){2-3} 			   	  			\cmidrule(lr){4-5} 								 	\cmidrule(lr){6-7}
Method     				& success   				& time  			& success   				& time 					& success   				& time 		\\
\midrule
DDPG+HER       		& 76.19\% 				& 7.33h 		& 20.32\%  				& 8.47h  				& 27.28\% 				& 7.55h		\\
DDPG+HER+PER 		& 75.46\% 				& 79.86h 		& 18.95\%  				& 80.72h  				& 27.74\% 				& 81.17h	\\
DDPG+HER+MEP   	& \textbf{81.30\%}	& 17.00h 		& \textbf{25.00\%}  & 19.88h  				& \textbf{31.88\%} 	& 25.36h		\\
\bottomrule
\end{tabular}
\label{tab:results}
\end{table*}


\section{Experiments}
\label{sec:experiments}

We test the proposed method on a variety of simulated robotic tasks, see Section~\ref{sec:environments}, and compare it to strong baselines, including DDPG and HER.
To the best of our knowledge, the most similar method to MEP is Prioritized Experience Replay (PER) \cite{schaul2015prioritized}. 
In the experiments, we first compare the performance improvement of MEP and PER. 
Afterwards, we compare the time-complexity of the two methods. 
We show that MEP improves performance with much less computational time than PER.
Furthermore, the motivations of PER and MEP are different. The former uses TD-errors, while the latter is based on an entropy-regularized objective function.

In this section, we investigate the following questions: 
\begin{enumerate}
\item Does incorporating goal entropy via MEP bring benefits to off-policy RL algorithms, such as DDPG or DDPG+HER?
\item Does MEP improve sample-efficiency of state-of-the-art RL approaches in robotic manipulation tasks?
\item How does MEP influence the entropy of the achieved goal distribution during training?
\end{enumerate}
Our code is available online at \url{https://github.com/ruizhaogit/mep.git}.
The implementation uses OpenAI Baselines \cite{baselines} with a backend of TensorFlow \cite{abadi2016tensorflow}.

\subsection{Performance}
To test the performance difference among methods including DDPG, DDPG+PER, and DDPG+MEP, we run the experiment in the three robot arm environments.
We use the DDPG as the baseline here because the robot arm environment is relatively simple. 
In the more challenging robot hand environments, we use DDPG+HER as the baseline method and test the performance among DDPG+HER, DDPG+HER+PER, and DDPG+HER+MEP.
To combine PER with HER, we calculate the TD-error of each transition based on the randomly selected achieved goals. 
Then we prioritize the transitions with higher TD-errors for replay. 

Now, we compare the mean success rates. 
Each experiment is carried out with 5 random seeds and the shaded area represents the standard deviation. 
The learning curve with respect to training epochs is shown in Figure \ref{fig:fig_accuracy}. 
For all experiments, we use 19 CPUs and train the agent for 200 epochs. 
After training, we use the best-learned policy for evaluation and test it in the environment. 
The testing results are the mean success rates. 
A comparison of the performances along with the training time is shown in Table \ref{tab:results}.

\begin{figure*}[!ht]
	\centering
	\includegraphics[width=5.5 in]{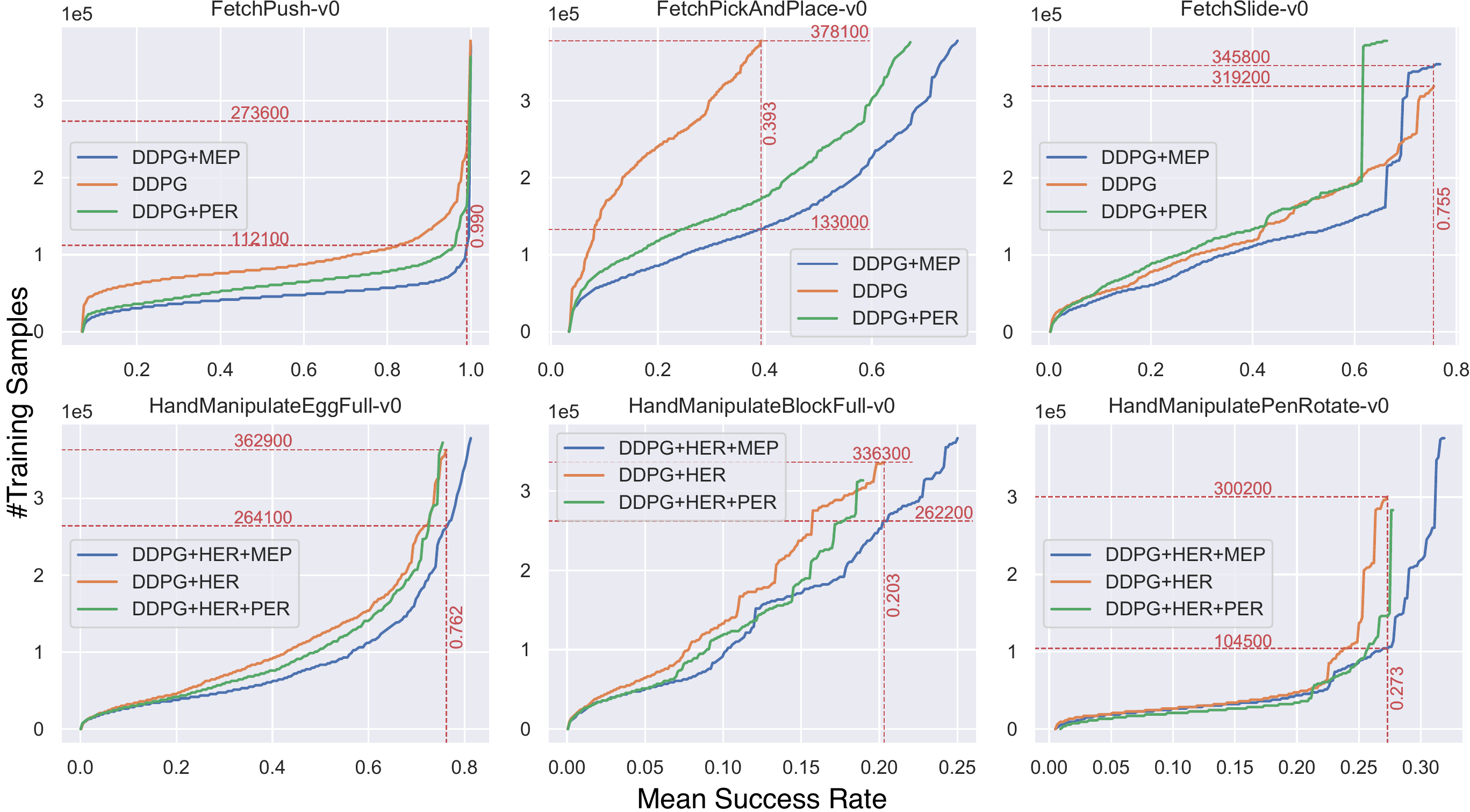}
	\caption{Number of training samples needed with respect to mean success rate for all six environments (the lower the better)}
	\label{fig:eff6plot}
\end{figure*}

From Figure \ref{fig:fig_accuracy}, we can see that MEP converges faster in all six tasks than both the baseline and PER. 
The agent trained with MEP also shows a better performance at the end of the training, as shown in Table \ref{tab:results}.
In Table \ref{tab:results}, we can also see that the training time of MEP lies in between the baseline and PER.
It is known that PER can become very time-consuming \cite{schaul2015prioritized}, especially when the memory size $N$ is very large.
The reason is that PER uses TD-errors for prioritization. After each update of the model, the agent needs to update the priorities of the transitions in the replay buffer, which is ${O}(\log{}N)$.  
In our experiments, we use the efficient implementation based on the ``sum-tree" data structure, which can be relatively efficiently updated and sampled from \cite{schaul2015prioritized}. 
To be more specific, MEP consumes much less computational time than PER. 
For example in the robot arm environments, on average, DDPG+MEP consumes about 1.2 times the training time of DDPG.
In comparison, DDPG+PER consumes about 5 times the training time as DDPG.
In this case, MEP is 4 times faster than PER.
MEP is faster because it only updates the trajectory density once per epoch and can easily be combined with any multi-goal RL methods, such as DDPG and HER.

Table \ref{tab:results} shows that baseline methods with MEP result in better performance in all six tasks. 
The improvement increases by up to 39.34 percentage points compared to the baseline methods. 
The average improvement over the six tasks is 9.15 percentage points. 
We can see that MEP is a simple, yet effective method and it improves state-of-the-art methods.

\subsection{Sample-Efficiency}
To compare sample-efficiency of the baseline and MEP, we compare the number of training samples needed for a certain mean success rate. 
The comparison is shown in Figure \ref{fig:eff6plot}.
From Figure \ref{fig:eff6plot}, in the \texttt{FetchPush-v0} environment, we can see that for the same 99\% mean success rate, 
the baseline DDPG needs 273,600 samples for training, while DDPG+MEP only needs 112,100 samples. 
In this case, DDPG+MEP is more than twice (2.44) as sample-efficient as DDPG. 
Similarly, in the other five environments, MEP improves sample-efficiency by factors around one to three. 
In conclusion, for all six environments, MEP is able to improve sample-efficiency by an average factor of two (1.95) over the baseline's sample-efficiency.

\subsection{Goal Entropy}
\begin{figure*}[!ht]
	\centering
	\includegraphics[width=5.5 in]{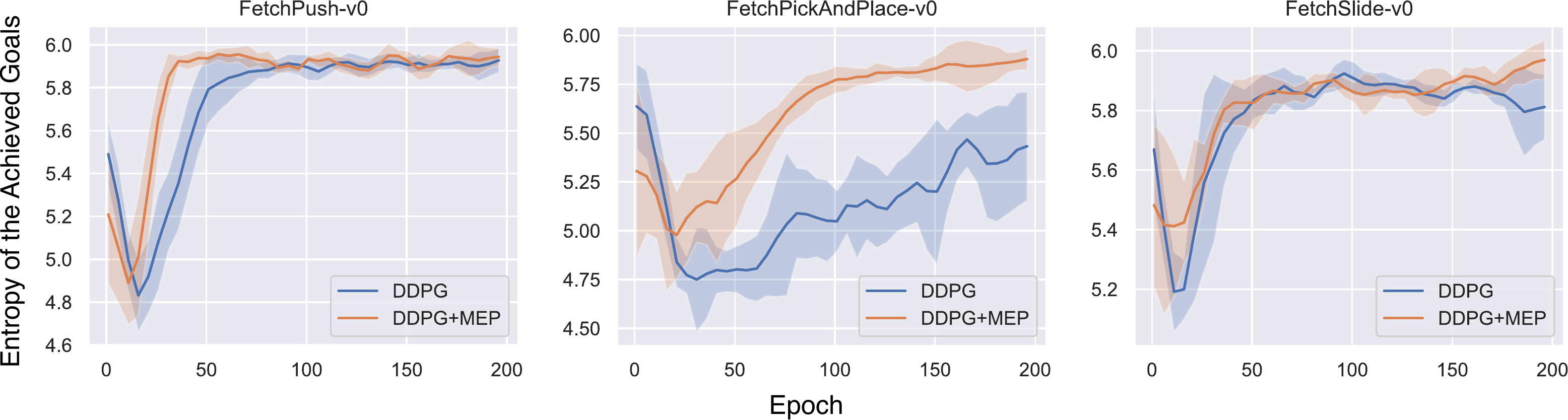}
	\caption{Entropy values of the achieved goal distribution $\ent_p (\Taub^g)$ during training}
	\label{fig:goal-entropy}
\end{figure*} 
To verify that the overall MEP procedure works as expected, we calculated the entropy value of the achieved goal distribution $\ent_p (\Taub^g)$ with respect to the epoch of  training. 
The experimental results are averaged over 5 different random seeds. 
Figure~\ref{fig:goal-entropy} shows the mean entropy values with its standard deviation in three different environments.
From Figure~\ref{fig:goal-entropy}, we can see that the implemented MEP algorithm indeed increases the entropy of the goal distribution. 
This affirms the consistency of the stated theory with the implemented MEP framework.


\section{Related Work}
\label{sec:related_work}
Maximum entropy was used in RL by \citet{williams1991function} as an additional term in the loss function to encourage exploration and avoid local minimums \cite{mnih2016asynchronous,wu2016training,nachum2016improving,asadi2016alternative}. A similar idea has also been utilized in the deep learning community, where entropy loss was used as a regularization technique to penalize over-confident output distributions \cite{pereyra2017regularizing}.  In RL, the entropy loss adds more cost to actions that dominate quickly. A higher entropy loss favors more exploration \cite{mnih2016asynchronous}. \citet{neu2017unified} gave a unified view on entropy-regularized Markov Decision Processes (MDP) and discussed the convergence properties of entropy-regularized RL, including TRPO \cite{schulman2015trust} and A3C \cite{mnih2016asynchronous}. 

More recently, \citet{haarnoja2017reinforcement} and \citet{levine2018reinforcement} proposed deep energy-based policies with state conditioned entropy-based regularization, which is known as Soft-Q Learning. They showed that maximum entropy policies emerge as the solution when optimal control is cast as probabilistic inference.
Concurrently, \citet{schulman2017equivalence} showed the connection and the equivalence between Soft-Q Learning and policy gradients. 
Maximum entropy policies are shown to be robust and lead to better initializations for RL agents \cite{haarnoja2018latent,haarnoja2018composable}.
Based on maximum entropy polices, \citet{eysenbach2018diversity} developed an information theoretic objective, which enables the agent to automatically discover different sets of skills. 

Unlike aforementioned works \cite{williams1991function,mnih2016asynchronous,haarnoja2017reinforcement}, the information theoretic objective \cite{eysenbach2018diversity} uses state, not actions, to calculate the entropy for distinguishing different skills. Our work is similar to this previous work \cite{eysenbach2018diversity} in the sense that we also use the states, instead of actions, to calculate the entropy term and encourage the trained agent to cover a variety of goal-states. Our method generalizes to multi-goal and multi-task RL \cite{kaelbling1993hierarchical,sutton1999between,bakker2004hierarchical,sutton2011horde,
szepesvari2014universal,schaul2015universal, pinto2017learning,plappert2018multi}. 

The entropy term that we used in the multi-goal RL objective is maximized over goal-states. 
We use maximum goal entropy as a regularization for multi-goal RL, which encourages the agent to learn uniformly with respect to goals instead of experienced transitions. This corrects the bias introduced by the agent's behavior policies. For example, the more easily achievable goals are generally dominant in the replay buffer. The goal entropy-regularized objective allows the agent to learn to achieve the unknown real goals, as well as various virtual goals.

We implemented the maximum entropy regularization via prioritized sampling based on achieved goal-states. 
We believe that the most similar framework is prioritized experience replay \cite{schaul2015prioritized}.
Prioritized experience replay was introduced by \citet{schaul2015prioritized} as an improvement to the experience replay in DQN \cite{mnih2015human}.
It prioritizes the transitions with higher TD-error in the replay buffer to speed up training.
The prioritized experience replay is motivated by TD-errors.
However, the motivation of our method comes from information theory--maximum entropy.
Compared to prioritized experience replay, our method performs superior empirically and consumes much less computational time.

The intuition behind our method is to assign priority to those under-represented goals, which are relatively more valuable to learn from (see Appendix).
Essentially, our method samples goals from an entropy-regularized distribution, rather than from a true replay buffer distribution, which is biased towards the behavior policies.
Similar to recent work on goal sampling methods \cite{forestier2017intrinsically,pere2018unsupervised,florensa2018automatic,zhao2018energy,nair2018visual,warde2018unsupervised}, our aim is to model a goal-conditioned MDP.
In the future, we want to further explore the role of goal entropy in multi-goal RL.
      

\section{Conclusion}
This paper makes three contributions. First, we propose the idea of Maximum Entropy-Regularized Multi-Goal RL, which is essentially a reward-weighted entropy objective.
Secondly, we derive a safe surrogate objective, i.e., a lower bound of the original objective, to achieve stable optimization. 
Thirdly, we implement a novel Maximum Entropy-based Prioritization framework for optimizing the surrogate objective.
Overall, our approach encourages the agent to achieve a diverse set of goals while maximizing the expected return.

We evaluated our approach in multi-goal robotic simulations.
The experimental results showed that our approach improves performance and sample-efficiency of the agent while keeping computational time under control.
More precisely, the results showed that our method improves performance by 9 percentage points and sample-efficiency by a factor of two compared to state-of-the-art methods.


\bibliography{reference}
\bibliographystyle{icml2019}

\onecolumn
\appendix

\section{Proof of Theorem 1}

\begin{theorem}
	\label{th:lower-bound}
	The surrogate $\eta^{\lb}(\bs{\theta})$ is a lower bound of the objective function $\eta^{\ent}(\bs{\theta})$, i.e.,  $\eta^{\lb}(\bs{\theta}) < \eta^{\ent}(\bs{\theta})$, where

	\begin{equation}
	\eta^{\ent}(\bs{\theta}) =\ent^w_p(\Taub^g) = \mathbb{E}_p\left[\log\frac{1}{p(\taub^g)} \sum_{t = 1}^T r(S_t, G^e) \mid \bs{\theta}\right] 
	\end{equation} 
	
	\begin{equation}
	\eta^{\lb}(\bs{\theta}) = Z \cdot \mathbb{E}_q \left[\sum_{t = 1}^T r(S_t, G^e) \mid \bs{\theta} \right]
	\label{eq:th1-lb}
	\end{equation}
	
	\begin{equation}
	q(\taub^g) = \frac{1}{Z}p(\taub^g)\left(1-p(\taub^g)\right)
	\end{equation}
	
	$Z$ is the normalization factor for $q(\taub^g)$.
	$\ent^w_p(\Taub^g)$ is the weighted entropy \cite{guiacsu1971weighted, Kelbert2017}, where the weight is the accumulated reward $\sum_{t = 1}^T r(S_t, G^e)$ in our case. 
	
\end{theorem}
\begin{proof}
	\begin{align} 
	\eta^{\lb}(\bs{\theta}) &=	
	Z \cdot \mathbb{E}_q \left[\sum_{t = 1}^T r(S_t, G^e) \mid \bs{\theta} \right] \\
	 &=\sum_{\taub^g} Z \cdot q({\taub}^g) \sum_{t = 1}^T r(s_t, g^e) \\
	&=\sum_{\taub^g} \frac{Z}{Z} p(\taub^g) (1-p(\taub^g)) \sum_{t = 1}^T r(s_t, g^e)\\
	&< \sum_{\taub^g} -p(\taub^g) \log p(\taub^g) \sum_{t = 1}^T r(s_t, g^e)\\
	&=\mathbb{E}_p\left[\log\frac{1}{p(\taub^g)} \sum_{t = 1}^T r(S_t, G^e) \mid \bs{\theta}\right] \\
	&= \ent^w_p(\Taub^g)\\
    &= \eta^{\ent}(\bs{\theta})
	\end{align} 
In the inequality, we use the property $\log x < x-1$.
\end{proof}

\section{Proof of Theorem 2}

\begin{theorem}
\label{th:higher-entropy}
Let the probability density function of goals in the replay buffer be 
\begin{equation}
p(\taub^g),\text{where}\ p(\taub^g_i) \in (0,1)\ \text{and}\ \sum_{i=1}^{N} p(\taub^g_i)=1.
\end{equation}
Let the proposal probability density function be defined as 
\begin{equation}
q(\taub^g_i) = \frac{1}{Z} p(\taub^g_i) \left( 1 - p(\taub^g_i) \right ) ,\ \text{where}\ \sum_{i=1}^N q(\taub^g_i)=1.
\end{equation}
Then, the proposal goal distribution has an equal or higher entropy 
\begin{equation}
\ent_q(\Taub^g) - \ent_p (\Taub^g) \geq 0.
\end{equation}
\end{theorem}

\begin{proof}

For clarity, we define the notations in this proof as $p_i = p(\taub^g_i)$ and $q_i = q(\taub^g_i)$.

Note that the definition of Entropy is 
\begin{equation}
\ent_p = \sum_i -p_i\log(p_i),
\end{equation}
where the $i$th summand is $p_i\log(p_i)$, which is a concave function. 
Since the goal distribution has a finite support $I$, we have the real-valued vector $(p_1, \hdots, p_N)$ and $(\frac{1}{Z}q_1, \hdots, \frac{1}{Z}q_N)$. 

We use Karamata's inequality \cite{kadelburg2005inequalities}, which states that if the vector $(p_1, \hdots, p_N)$ majorizes $(\frac{1}{Z}q_1, \hdots, \frac{1}{Z}q_N)$ then the summation of the concave transformation of the first vector is smaller than the concave transformation of the  second vector. 

In our case, the concave transformation is the weighted information at the $i$th position -$p_i\log(p_i)$, where the weight is the probability $p_i$ (entropy is the expectation of information).
Therefore, the proof of the theorem is also a proof of the majorizing property of $p$ over $q$ \cite{entrostack}. 

We denote the proposal goal distribution as 
\begin{equation}
q_i = f(p_i) = \frac{1}{Z}p_i(1-p_i).
\end{equation}
Note that in our case, the partition function $Z$ is a constant. 

Majorizing has three requirements \cite{marshall1979inequalities}.

The first requirement is that both vectors must sum up to one.
This requirement is already met because 
\begin{equation}
\sum_i p_i = \sum_iq_i = 1.
\end{equation}

The second requirement is that monotonicity exits.
Without loss of generality, we assume the probabilities are sorted: 
\begin{equation}
p_1 \geq p_2 \geq \hdots \geq p_N
\end{equation}
Thus, if $i > j$ then 
\begin{align}
f(p_i) - f(p_j) &= \frac{1}{Z}p_i(1-p_i) - \frac{1}{Z}p_j(1-p_j) \\
					&=  \frac{1}{Z}[(p_i - p_j) - (p_i + p_j)(p_i - p_j)] \\
					&= \frac{1}{Z}(p_i - p_j) (1-p_i-p_j) \\
					&\geq 0.
\end{align}
which means that if the original goal probabilities are sorted, the transformed goal probabilities are also sorted,
\begin{equation}
f(p_1) \geq f(p_2) \geq \hdots \geq f(p_N).
\end{equation}

The third requirement is that for an arbitrary cutoff index $k$, there is
\begin{equation}
p_1 + \hdots p_k < q_1 + \hdots + q_k.
\end{equation}
To prove this, we have
\begin{align}
p_1 + \hdots + p_k  & = \frac{p_1 + \hdots + p_k }{1} \\
							   &= \frac{p_1 + \hdots + p_k }{p_1 + \hdots + p_N} \\
							   &\geq f(p_1) + ... + f(p_k) \\
							   &= \frac{1}{Z}[p_1(1-p_1) + ... + p_k(1-p_k)] \\
							   &= \frac{1}{Z}[p_1 + \hdots + p_k - (p_1^2 + \hdots + p_k^2)]
\end{align}

Note that, we multiply $Z*1$ to each side of
\begin{equation}
Z = p_1(1-p_1) + \hdots + p_N(1-p_N).
\end{equation}
Then we have
\begin{equation}
(p_1 + \hdots  + p_k)Z * 1 \geq p_1 + \hdots + p_k - (p_1^2 + \hdots + p_k^2) * 1.
\end{equation}

Now, we substitute the expression of $Z$ and then have
\begin{equation}
(p_1 + \hdots  + p_k)[p_1(1-p_1) + \hdots + p_N(1-p_N)] \geq [p_1 + \hdots + p_k - (p_1^2 + \hdots + p_k^2)] * 1.
\end{equation}

We express $1$ as a series of terms $\sum_i p_i$, we have
\begin{equation}
(p_1 + \hdots  + p_k)[p_1(1-p_1) + \hdots + p_N(1-p_N)] \geq [p_1 + \hdots + p_k - (p_1^2 + \hdots + p_k^2)] * [(p_1 +\hdots + p_k ) + (p_{k+1} + \hdots p_{N})].
\end{equation}

We use the distributive law to the right side and have
\begin{align*}
&(p_1 + \hdots  + p_k)[p_1(1-p_1) + \hdots + p_N(1-p_N)] \\
\geq &[p_1 + \hdots + p_k ] * [(p_1 +\hdots + p_k ) + (p_{k+1} + \hdots p_{N})] - [(p_1^2 + \hdots + p_k^2)]* [(p_1 +\hdots + p_k ) + (p_{k+1} + \hdots p_{N})].
\numberthis
\end{align*}

We move the first term on the right side to the left and use the distributive law then have
\begin{equation}
(p_1 + \hdots  + p_k)[-1*(p_1^2 + \hdots + p_N^2))] \geq  - [(p_1^2 + \hdots + p_k^2)]* [(p_1 +\hdots + p_k ) + (p_{k+1} + \hdots p_{N})].
\end{equation}

We use the distributive law again on the right side and move the first term to the left and use the distributive law then have
\begin{equation}
(p_1 + \hdots  + p_k)[-1*(p_{k+1}^2 + \hdots + p_N^2))] \geq  - [(p_1^2 + \hdots + p_k^2)]* [ (p_{k+1} + \hdots p_{N})].
\end{equation}

We remove the minus sign then have
\begin{equation}
(p_1 + \hdots  + p_k)[(p_{k+1}^2 + \hdots + p_N^2))] \leq   [(p_1^2 + \hdots + p_k^2)]* [ (p_{k+1} + \hdots p_{N})].
\end{equation}

To prove the inequality above, it suffices to show that the inequality holds true for each associated term of the multiplication on each side of the inequality.

Suppose that
\begin{equation}
i \leq k < j
\end{equation}
then we have
\begin{equation}
p_i > p_j.
\end{equation}
As mentioned above, the probabilities are sorted in descending order. We have 
\begin{equation}
p_ip_j^2- p_i^2p_j = p_ip_j(p_j - p_i) <0
\end{equation}
then
\begin{equation}
p_ip_j^2 < p_i^2p_j.
\end{equation}

Therefore, we have proved that the inequality holds true for an arbitrary associated term, which also applies when they are added up.
\end{proof}

\section{Insights}

\begin{figure*}[!ht]
	\centering
	\includegraphics[width=5. in]{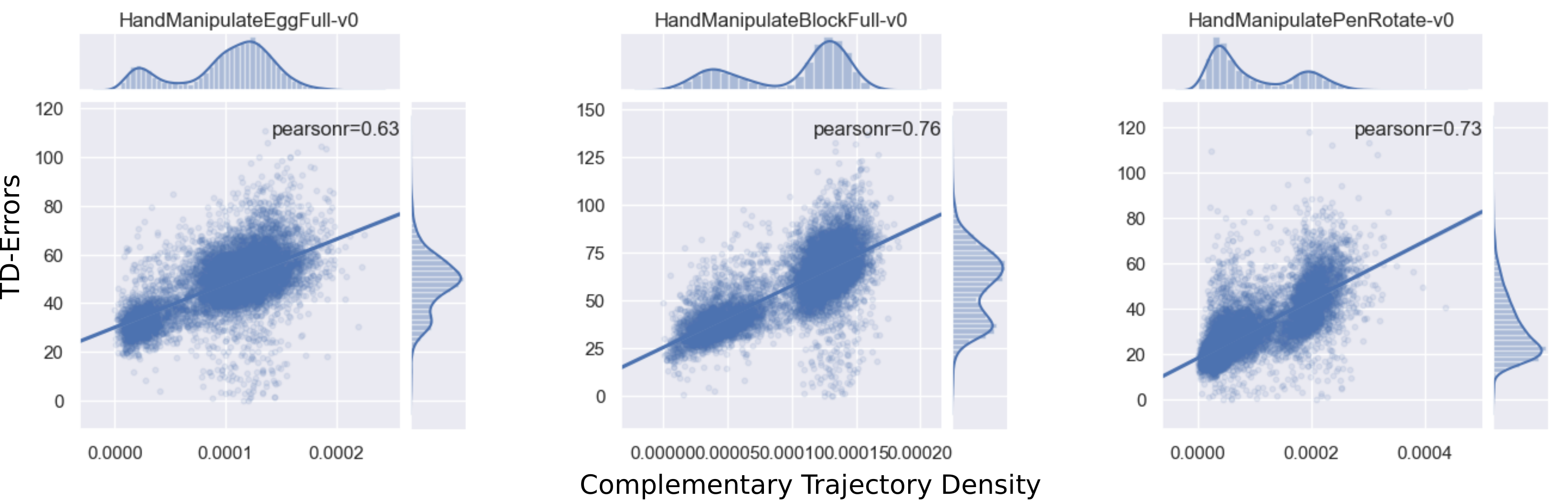}
	\caption{Pearson correlation between the complementary density $\bar{p}(\taub^g) $ and TD-errors in the middle of training}
	\label{fig:pearson3plot}
\end{figure*} 
To further understand why maximum entropy in goal space facilitates learning, we look into the TD-errors during training.
We investigate the correlation between the complementary predictive density $\bar{p}(\taub^g \mid \bs{\phi})$ and the TD-errors of the trajectory.
The Pearson correlation coefficients, i.e., Pearson's r \cite{benesty2009pearson}, between the density $\bar{p}(\taub^g \mid \bs{\phi})$ and the TD-errors of the trajectory are 0.63, 0.76, and 0.73, for the hand manipulation of egg, block, and pen tasks, respectively.
The plot of the Pearson correlation is shown in Figure~\ref{fig:pearson3plot}.
The value of Pearson's r is between 1 and -1, where 1 is total positive linear correlation, 0 is no linear correlation, and -1 is total negative linear correlation.
We can see that the complementary predictive density is correlated with the TD-errors of the trajectory with an average Pearson's r of 0.7.
This proves that the agent learns faster from a more diverse goal distribution.
Under-represented goals often have higher TD-errors, and thus are relatively more valuable to learn from.
Therefore, it is helpful to maximize the goal entropy and prioritize the under-represented goals during training. 

\end{document}